\documentclass[a4paper, 11pt]{article}
\usepackage[utf8]{inputenc}
\usepackage{amssymb}
\usepackage{amsmath}
\usepackage{amsthm}
\usepackage{tikz}
\usepackage{subfig}
\usepackage{xcolor}
\usepackage{pgfplots}
\usetikzlibrary{arrows}
\usepackage{filecontents}
\definecolor{midnightblue}{rgb}{0.4,0,0.6}
\definecolor{limegreen}{rgb}{0,0.7,0}
\usepackage{pgfplotstable}
\usetikzlibrary{external}
\usepackage{cleveref}
\usepackage{bbold}
\usepackage{url}
\usepackage{authblk}
\usetikzlibrary{plotmarks}
\usetikzlibrary{patterns}

\let\emptyset\varnothing

\usetikzlibrary{shapes}
\usetikzlibrary{positioning}

\newtoks\bsubfloattoks
\newdimen\bsubfloatht

\makeatletter

\newcommand{\bsubfloat}[2][]{%
  \sbox\z@{#2}%
  \ifdim\bsubfloatht<\ht\z@
    \bsubfloatht=\ht\z@
  \fi
  \bsubfloattoks=\expandafter{\the\bsubfloattoks
    \bsubfloatspace\subfloat[#1]{\vbox to\bsubfloatht{\hbox{#2}\vfill}}}%
}
\newcommand\resetbsubfloatrows{\bsubfloatht\z@\bsubfloattoks={\@gobble}}
\newcommand{\printbsubfloatrow}{\the\bsubfloattoks}

\newenvironment{customlegend}[1][]{%
        \begingroup
        \csname pgfplots@init@cleared@structures\endcsname
        \pgfplotsset{#1}%
    }{%
        \csname pgfplots@createlegend\endcsname
        \endgroup
    }%
    \def\addlegendimage{\csname pgfplots@addlegendimage\endcsname}
\makeatother
\newcommand{\personnalheight}{.2\textheight}
\newcommand{\personnalwidth}{.5\textwidth}

\newcommand{\smallwidth}{.4\textwidth}

\newcommand{\wrt}{\textit{w.r.t.}{}}

\newtheorem{definition}{Definition}
\newtheorem{theorem}{Theorem}

\crefname{figure}{figure}{figures}
\crefname{equation}{equation}{equations}
\crefname{proposition}{theorem}{theorems}

\title{Curiosity-Aware Bargaining}
\author[1,2]{Cédric L.R. \textsc{Buron}
\thanks{Corresponding Author: Cédric Buron, UPMC
cedric.buron@lip6.fr\\We thank Finance Innovation, ECM2, nov@log and the Région Île de France for supporting our project.}}
\author[2]{Sylvain \textsc{Ductor}}
\author[2,3]{Zahia \textsc{Guessoum}}

\affil[1]{Kyriba, Saint Cloud, France}
\affil[2]{Universit\'e Pierre et Marie Curie, Paris, France}
\affil[3]{Université de Reims Champagne Ardennes}

\begin{document}
\pagestyle{headings}
\def\thepage{}

\maketitle
\begin{abstract}
Opponent modeling consists in modeling the strategy or preferences of an agent thanks to the data it provides.
In the context of automated negotiation and with machine learning, it can result in an advantage so
overwhelming that it may restrain some casual agents to be part of the bargaining.
We qualify as ``curious'' an agent driven by the desire of negotiating in order to collect information and improve its opponent model.
However,  neither curiosity-based rationality nor curiosity-robust protocol have been studied in automatic negotiation.
In this paper, we rely on mechanism design to propose three extensions of the standard bargaining protocol that limit information leak.
Those extensions are supported by an enhanced rationality model, that considers the exchanged information. Also, they are theoretically analyzed and experimentally evaluated.
\end{abstract}

\pagestyle{empty}

\section{Introduction}

Negotiation involves several parties that exchange proposals in order to reach a mutually beneficial outcome.
Each participant
has private information and a utility function that may be in conflict with other participants.

To get a good agreement for herself, a participant may choose to use opponent modeling (see \cite{survey} for a survey on this topic).
Opponent modeling is based on machine learning techniques, which are fed by the sequences of exchanged proposals.
It allows to learn the way the opponent values goods
(\cite{carmel1996opponent,oshrat2009facing}), which provides a major advantage in the future negotiations (\cite{VanBragt2003Why}).

Opponent modelers have therefore an incentive to pretend to negotiate to get a good while they are seeking as much data as possible from their opponent. Opponent modelers that act this way are called ``curious''.
Curious agents have a harmful influence on negotiations: collected information
provides them more efficient modelings, which will significantly unbalance the future negotiations in their favor.
Consequently, casual agents may be discouraged and be reluctant to use the bargaining protocol.

In this paper, we design a protocol that extends the standard bargaining in order to restrain the sphere of operation of curious agents.
To do that, we first present related work (\Cref{Sec:RelWork}) and model how the agents evaluate exchanged information by extending the classical model of rationality (\Cref{Sec:Model}). 
We then improve standard bargaining with three theoretically analyzed extensions (\Cref{Sec:Contrib}). The first one prevents agents from initiating the negotiation if there is no chance to reach an agreement. The second one directly limits the number of exchanged proposals. The third one prohibits the agents from leaving the negotiation without agreement. The significance of our proposed extensions is proved empirically (\Cref{Sec:XP}). We finally summarize our proposal and give some possible extensions to our work (\Cref{Sec:Concl}).

\section{Related Work}
\label{Sec:RelWork}

In bargaining \cite{Palgrave2008}, each party tries to reach a joint decision involving two conflicting goals:
maximizing the utility of an agreement and maximizing the chance of acceptance by the
opponent. Relevant tactics have been proposed to achieve the first goal such as the efficient Tit-for-Tat~\cite{Hindriks2009benefits}.
The second goal requires skills such as curiosity to understand the
opponents' strategies and preferences. Psychologists Fisher {\it et al.} \cite{Fisher} define the curiosity
as the ability to see the situation as the other side sees it.

In automated negotiation, this ability corresponds to what is called ``opponent modeling'' \cite{survey}.
We distinguish two approaches: on single session (online) and on multi-session (offline).
Among the work on single session modeling,
the most widely used techniques are Bayesian learning \cite{Hindriks2008Opponent}, non-linear regression \cite{Hou2004Modelling}, and SVM \cite{Laviers2009Opponent}. Multi-session techniques include Kernel Density Estimation \cite{oshrat2009facing} and neural networks \cite{Carbonneau2008Predicting}.
The preponderance of this approach is attested by the Automated Negotiation Agent Competition \cite{ANAC} since all high ranking agents make a wide use of it \cite{Baarslag2015Automated}.

The benefits of opponent modeling can lead to manipulative behavior where agents try to get information at any cost. Also, the negotiation ends up at the expense of the agents that do not mean to invest in this technique. Bargaining being vulnerable to those behaviors, this protocol is disadvantageous for the latters and therefore unreliable in general case.
It is thus necessary to develop a new bargaining framework resistant to those practices. Such a topic belongs to the theory of mechanism design.

Mechanism design \cite{Myerson1989Mechanism,Maskin2008Mechanism} studies the conception of ``incentive-compatible'' mechanisms
\cite{Hurwicz1972informationally}, \textsl{i.e.}  protocols in which the agents act according to the desire of the designer of the protocol.
As an example,
in the context of mediated bargaining\footnote{in mediated bargaining, the negotiators use a trusted third party called mediator. The latter uses information delivered by the negotiators to generate the proposals.}, several work ({\it e.g.} \cite{Myerson1979Incentive,Chatterjee1982Incentive}) give a theoretical framework on how to design a protocol in which the agents provide their real reserve prices
to the mediator. 

To the best of our knowledge, there is no work on mechanism design related to standard (non-mediated) bargaining. In particular, there is no attempt to design an extension of bargaining resistant to the aforementioned problem.

\section{Model of Curiosity-Aware Bargaining} \label{sec:bargain}

\label{Sec:Model}

We introduce in this section an original rationality model for bargaining.
Agents consider not only the price associated to the good but also the leaked and collected information.
We first introduce a model of the bargaining protocol that explicitly represents this information.
Then we extend the two affected objects: the agent utility and the reserve price.
Finally we propose a classification of curiosity-aware agents.

\subsection{Bargaining Model}\label{sec:bargain:model}

Bargaining is a type of negotiation in which two agents, a \emph{purchaser} $p$ and a \emph{seller} $s$, 
seek mutual agreement on the price of a good.
Each agent is endowed with an \textbf{agreement domain} that induces, for the seller, the minimum price it is willing to sell the good, and for the purchaser, the maximum price it is willing to pay for it.
Those values are called  \textit{reserve prices} ; an agent has a null utility if the agreed price is exactly the reserve price, 
a positive utility if it is better and a negative one if it is worse.

A bargaining proceeds with a \textbf{a sequence of exchanged proposals} which respect the monotonic concession principle: in any two successive proposals of an agent the second is always better for it opponent than the first. 
 A bargaining \textbf{ends either by an accept or a reject} by one of the agents. The reject may be sent at any time to end the negotiation without any transaction.

\begin{definition}[Bargaining]\label{def:barg}
A bargaining proceeds with exchange of proposals and ends with an accept or a reject of one of the participants. 
We note $V$ the space of lists of messages sent by one party during a bargaining.
\begin{equation}
 V = \{(k,p,e)| k \in \mathbb N, p \in \mathbb{R}^{k}, e \in \{\emptyset,\top,\bot\}\}
\end{equation}
where
\begin{description}
 \item[$k \in \mathbb N$] is the number of proposals
 \item[$p \in \mathbb{R}^{k}$] is the list of proposed prices
 \item[$e \in \{\emptyset,\top,\bot\}$] is the ending message: either $\top$ if this party accepts the last deal, 
$\bot$ if it rejects it or $\emptyset$ if the other party ends the negotiation.
\end{description}
We note $\mathcal{B}$ the set of bargaining records. A bargaining $b \in \mathcal B$ is defined with a six-tuple: 
 \begin{equation}
 b = (g_b,s_b,p_b,\pi_b,\nu^s_b,\nu^p_b)
 \end{equation}
 where
 \begin{description}
  \item[$g_b \in \mathcal Good$] is the negotiated good, 
\item[$s_b, p_b \in \mathcal Agent^2$] are respectively the seller and the purchaser identifiers
  \item[$\pi_b \in \mathbb R^+ \cup \{\bot\}$] is the negotiated price if the negotiation succeeds or $\bot$ if it fails.
\item[$\nu^s_b, \nu^p_b \in V^2$] are respectively the sequence of the seller and purchaser proposals.
 \end{description}
\end{definition}

\subsection{Curiosity-Aware Utility \& Reserve Price}\label{sec:discussion}
In the classical rationality model used in bargaining, an agent is only interested in the final price of the good. A curiosity-aware agent is also interested in the information exchanged during the bargaining process. This affects both its utility model and its reserve price. 

 \begin{definition}[Curiosity-Aware Utility]
The utility of a curiosity-aware agent is computed from the final price and the sequence of exchanges.
\begin{equation}
\forall a \in Agent, \forall b=(g_b,s_b,p_b,\pi_b,\nu^s_b,\nu^p_b)\in\mathcal{B},
\hfill\begin{array}[t]{llll}
u_a :& \mathbb R \times V \times V\  & \to & \mathbb R \\
     & (\pi_b, \nu^s_b,   \nu^p_b) & \mapsto & u_a(\pi_b, \nu^s_b,   \nu^p_b)
\end{array}\hfill
\end{equation}
 \end{definition}

The reserve price is the price for which the agent utility is null.
Classical agents have a constant reserve price throughout the negotiation
, 
while the reserve price of curiosity-aware agents is affected by the proposals made so far.

\begin{definition}[Curiosity-Aware Reserve Price]
Let $b$ be a bargaining. We note $\pi^p_b(k) \in \mathbb R$ the reserve price of the purchaser at the intermediary step $k$.
It is defined by the following property:
\begin{equation}\label{eq:equilibre}
\begin{aligned}
\forall  b \in \mathcal B, \forall k \leq k^p_b, u_p(\pi^p_b(k), \nu^s_{\setminus k}, \nu^p_{\setminus k}) = 0
\end{aligned}
\end{equation}
where $k^p_b$ is the number of proposals made by purchaser $p$ (see \Cref{def:barg}), $\nu^p_{\setminus k}$ 
holds the $k$ first proposals of the list $\nu^p$ and $\nu^s_{\setminus k}$ 
holds the $k$ first proposals of the list $\nu^s$ (if the seller made the first proposal, $k-1$ else).
 \end{definition}
  
\subsection{Agent Taxonomy}

We now introduce three canonical types of agents that are characterized by the way they value information and price.
For the sake of simplicity, we rely on comparators that are indifferently used for sellers and purchasers.
For Agent $a$:
\begin{description}
 \item[$\succ^\pi_a$ and $\sim^\pi_a$] are respectively strict and equivalence orders over prices. 
 A seller prefers a higher price whereas a purchaser prefers a lower one.
 \item[$\succ^\nu_a$ and $\sim^\nu_a$] are respectively strict and equivalence orders 
 over sequences of proposals (\textit{i.e.} $V$).
 For two sequences of proposals $\nu$ and $\nu'$, $\nu \succ^\nu_a \nu'$
 means that agent $a$ considers that $\nu$ carries more information than $\nu'$
\end{description}

In order to consistently extend the classical bargaining rationality model, we assume that, 
any other thing being equal (\textit{ceteris paribus}), 
any agent prefers a bargaining resulting in a better price outcome.
In other words,
if two bargainings reveal the same amount of information, an agent $a$ prefers the one that optimizes the price. 
The following equation is verified by any rational agent:
 \begin{equation} \label{eq:rational}
 \forall b,b' \in \mathcal B^2,
 \left\{\begin{array}{l}
   \pi_b \succ^\pi_a \pi_{b'} \\
 \nu^s_b \sim^\nu_a \nu^s_{b'}\\
 \nu^p_b \sim^\nu_a \nu^p_{b'}
       \end{array}
 \right. \implies u_a(\pi_b, \nu^s_b,   \nu^p_b) > u_a(\pi_{b'}, \nu^s_{b'},   \nu^p_{b'})
 \end{equation}

Curiosity-aware agents may also consider both the leaked and revealed information.
An agent can be \textit{secretive} or \textit{unsecretive} depending on whether or not it is concerned about the information it reveals. Also, an agent can be \textit{curious} or \textit{uncurious} depending on whether 
or not it is interested in the information it collects. 
\Cref{fig:reservepr:all} presents examples of reserve price functions for the three types of agents we detail below.
They are drown from \Cref{util}, used in our experiments.

\begin{figure}[h!]
    \centering
    \pgfplotsset{every axis legend/.append style={at={(0.97,0.03)},anchor=south east}}
 \begin{tikzpicture}
 \begin{axis}[
xlabel={$k$},
ylabel={$\pi$},
ymin=2, ymax=25,
xmin=0, xmax=40,
minor y tick num=1,
legend entries = {uncurious reserve price, secretive reserve price, curious reserve price, purchaser proposals},
legend style = {at={(1.5,0)}, anchor=south west},
width=\smallwidth,
xtick={0,10,20,30,40},
ylabel near ticks,
xlabel near ticks,
]
\newcommand\pie{15}
\newcommand\kpa{0.1}
\newcommand\numax{10}
\newcommand\bta{4}
\newcommand\gma{2}
\addplot[
color = blue,
domain=0:40,
samples=200,
dashdotted,
thick,
]
{15+0*x};
\addplot[
color = red,
domain=0:40,
samples=200,
densely dotted,
thick,
]
{15*ln(100000)/ln(100000+100000*x)};
\addplot[
color = violet,
domain=0:40,
samples=200,
dashed,
thick,
]
{15*ln(100000*x+100000)/ln(100000)};

\addplot[
domain=0:39,
samples=14,
only marks,
mark=+,
color=black,
very thick
]
{\gma + (\pie-\gma)*(\kpa+(1-\kpa)*(x/\numax)^(1/\bta))};

\addplot[
domain=-10:-8,
samples=10,
only marks,
mark=x,
color=black,
very thick
]
{0};
\end{axis}
\end{tikzpicture}
    \caption{Proposals and reserve price of uncurious, curious and secretive agents}
    \label{fig:reservepr:all}
\end{figure}
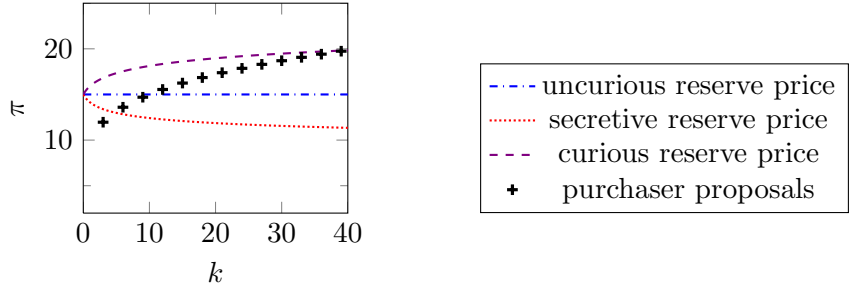
%

\paragraph{An Uncurious (unsecretive) agent}
implements the classical understanding of rationality in bargaining: its sole interest is the final price.

 \begin{definition}[Uncurious Agent]\label{def:uncurious}
An \textit{uncurious} agent utility is not affected by the exchanged information.
  \begin{equation}  
  \forall b,b' \in \mathcal B^2,
 \pi_b \succ^\pi_a \pi_{b'} \implies u_a(\pi_b, \nu^s_b,   \nu^p_b) > u_a(\pi_{b'}, \nu^s_{b'},\nu^p_{b'}) 
  \end{equation}
 \end{definition}

\paragraph{A Secretive (uncurious) agent}  is aware that its actions reveal information which can be used by malicious agents at its expense.
It is however only interested in selling/purchasing goods and not in collecting information about other agents.
The price it is willing to pay for a good decreases while it is making proposals. 
Thus, it tends to drop the negotiation sooner than a classical (uncurious) agent.
Also, if the deal is not concluded, its utility is negative since it reveals information and does not get any outcome.

 \begin{definition}[Secretive Agent] \label{def:secretive}
A \textit{secretive} agent prefers not to reveal information, but is not interested by the collected information. In the case of a purchaser $p$:
  \begin{equation}  \label{eq:secretive}
 \forall b,b' \in \mathcal B^2,
 \left\{\begin{array}{l}
 \pi_b \sim^\pi_p \pi_{b'}    \\
 \nu^p_{b'} \succ^\nu_p \nu^p_b 
 \end{array}
 \right. \implies u_p(\pi_b, \nu^s_b,\nu^p_b) > u_p(\pi_{b'}, \nu^s_{b'},\nu^p_{b'})
  \end{equation}
 \end{definition}

\paragraph{A Curious (unsecretive) agent}  is interested in collecting information that it may use for opponent modeling.
It is not concerned about the information it may reveal.
The price it is willing to pay for a good increases when the amount of collected information increases. 
It has a positive utility in case of reject.
A good strategy for such an agent is to make the bargaining end as late as possible and to eventually reject it.
\begin{definition}[Curious Agent] \label{def:curious}
 A \textit{curious} agent aims collecting as much information as possible during the exchange, but is not concerned by leaked information. 
 In the case of a purchaser $p$:
  \begin{equation}  \label{eq:curious}
 \forall b,b' \in \mathcal B^2,
 \left\{\begin{array}{l}
 \pi_b \sim^\pi_p \pi_{b'}    \\
 \nu^s_b \succ^\nu_p  \nu^s_{b'}
 \end{array}
 \right. \implies u_p(\pi_b, \nu^s_b,\nu^p_b) > u_p(\pi_{b'}, \nu^s_{b'},\nu^p_{b'}) 
  \end{equation}
\end{definition}

\paragraph{A Curious and Secretive agent} is both doing opponent modeling and aiming to preserve itself from being modeled by other agents.
It implements  \Cref{eq:curioussecretive} which is a more constrained version of    \Cref{eq:secretive,eq:curious} (noted $(\ref{eq:secretive}')$ and $(\ref{eq:curious}')$). It is interested  by either collecting more information \textit{ceteris paribus}, or revealing less information \textit{ceteris paribus}.
 In the case of a purchaser $p$:
  \begin{equation}  \label{eq:curioussecretive}
 \forall b,b' \in \mathcal B^2,
 \overset{(\ref{eq:secretive}')}{\left\{\begin{array}{l}
 \pi_b \sim^\pi_p \pi_{b'}    \\
 \nu^s_b \sim^\nu_p  \nu^s_{b'} \\
 \nu^p_{b'} \succ^\nu_p \nu^p_b 
 \end{array}\right.}
 \;\;
 \vee
 \;\;
 \overset{(\ref{eq:curious}')}{\left\{\begin{array}{l}
 \pi_b \sim^\pi_p \pi_{b'}    \\
 \nu^p_b \sim^\nu_a  \nu^p_{b'} \\
 \nu^s_b \succ^\nu_a  \nu^s_{b'}
 \end{array} \right.}
 \implies u_p(\pi_b, \nu^s_b,\nu^p_b) > u_p(\pi_{b'}, \nu^s_{b'},\nu^p_{b'}) 
  \end{equation}
 
We highlight the fact that the three previously introduced agent types can be used as atomic 
building blocks of this last type.
Indeed, during an exchange of messages, at each step,  a curious and secretive agent acts like one and only one canonical type.
For instance, when comparing two negotiations where $(\ref{eq:secretive}')$ (resp. $(\ref{eq:curious}')$)  applies, 
its utility is similar to the one of a secretive (resp. curious) agent. In the other numerous cases where \Cref{eq:curioussecretive} is not applicable,  it can be abstracted to a secretive agent each time it considers that the leaked information is more critical than the collected one, as a curious each time it values more the collected information than the revealed one, or as an uncurious if it estimates than the revealed and collected information compensate each other.

\section{Protocols for Curiosity-Proof Bargaining}
 \label{Sec:Contrib}
 
To limit the impact of curious agents, we introduce and analyze variants of the standard bargaining protocol.

We first analyze the different flaws of classical bargaining \wrt\ curious agents and then introduce a curiosity-proof extension for each phase of the protocol.
Last, we theoretically analyze some induced incentives.

As described in \Cref{sec:bargain:model}, bargaining can be analyzed in term of three aspects:
(1) the agreement domains, (2) the sequence of proposals and (3) the way the negotiation ends (accept or reject).
Each of those aspects may lead to situations with unnecessary leak of information, 
which may be exploited by curious agents.
\begin{description}
 \item[The mutual agreement domain] is the range of prices where a mutual agreement can be found. 
 If it is empty (the seller minimal price is higher than the purchaser maximal price) 
 the bargaining is not accepted: the revealing of information is useless. 
 Furthermore, a curious seller may decide of a reserve price much higher than its real one. 
 If the bargaining is concluded, it would collect information 
 and sell the good at a very high price, which is very good for it; otherwise it would collect information for free.
 \item[The sequence of proposals] is defined by the agent strategies and reserve prices.  
 Adopting strategies that make smartly chosen concessions 
 can increase the number of exchanges. Such strategies can be used by a curious agent 
 in order to optimize the collected information. 
 \item[The negotiation issue (accept or reject)] is also a crucial aspect. 
A curious agent gets a positive utility as soon as some proposals are done, 
even if the negotiation ends with a reject. 
\textit{A contrario}, in the case of a reject, the utility of a secretive agent is negative. 
This puts the ratio of power strongly in favor of the curious agents who may accept 
only if the proposed price is really advantageous.
\end{description}

We now present a variant for each of the steps that aims to reduce curiosity.

\subsection{Description of the proposed modifications}

\paragraph{Matching}
We first propose to avoid a useless leak of information resulting from bargaining between 
two parties that have incompatible agreement domains.
A trusted third-party collects the agents (initial) reserve prices and authorizes the bargaining if and only if a mutual agreement is possible. 

\paragraph{Bounds}
Secondly, we propose to bound the number of exchanges ; if no agreement has been reached by the bound, the bargaining fails. 
The bound associated to a bargaining $b \in \mathcal B$ is noted $k^{max}_b$.
Gneezy \textit{et al.} \cite{Gneezy2003Bargaining} shows that this advantages the  agent making the last proposal ; similarly to  \cite{Zhang2016Automated} we rely on a trusted third-party 
which transmits the proposals simultaneously, after both agents have made them. 

\paragraph{Enforcing the execution of a transaction}
Lastly, we propose to prevent the possible failure of a bargaining by enforcing 
the transaction even if the deal has been rejected or the bound of proposals has been reached. 
In this case, the good is exchanged at the agents reserve prices and the remaining money is taken by the community as a penalty.
This transaction is the less favorable for both parties and thus provides a strong incentive to conclude all the deals.

\subsection{Some proven properties}

We now prove two properties about the incentives induced by our proposals.
First, \Cref{theo:payment:matching} shows that the proposals provide a strong incentive, 
for a purchaser (resp. seller), not to declare a too high (resp. too low) reserve price. This means that, for instance, a curious agent will not lie on its reserve price in order to match as many agents as possible and cancel the negotiation after it has got enough information. Note that the question of providing an incentive, for a purchaser (resp. seller) not to declare a reserve price lower (resp. higher) than the real one is still open.
Hence, it is possible to manipulate the matching procedure in order to detect adversary parties with advantageous reserve prices and with the guarantee of getting a profitable deal.

\begin{theorem}\label{theo:payment:matching}
 Suppose the agent considers that the amount of information increases with the number of exchanged messages. Then, using the three extensions gives a curious purchaser (resp. seller) an incentive to give a  reserve price inferior (resp. superior) to $\pi^p_b(k_b^{max})$ (resp. $\pi^s_b(k_b^{max})$).
\end{theorem}
\begin{proof}
Without loss of generality, we suppose that our agent is a curious purchaser $p_b$. 
Let $b$ be a bargaining (see \Cref{def:barg}) between $p_b$ and a seller $s_b$. 
We prove that if the purchaser declares a more interesting price (for the seller) than its reserve price computed at the bound, either it pays more (case 1 and 2) or it is matched with the same seller as if it was honest (case 3).
\begin{description}
 \item[Case 1: $\pi_b = \bot$], the deal is rejected. In this case, due to the third extension, the agent pays the price it has declared. Hence, the lower the price, the better the deal. In particular, if the bound is reached, $p_b$ has to pay price higher than its reserve price and gets a negative utility.
 \item[Case 2: $\pi_b \in \mathbb{R} \wedge \pi_b > \pi^p_b(k^{max}_b)$], the bargaining is accepted, but $p_b$ pays a price higher than his true reserve price. This case may arise if the last proposal before the bound is higher than its true reserve price, but lower than the declared one. It is then more interesting for $p_b$ to accept the proposal than to reach the bound and pay the declared reserve price. Such a situation can happen whether the seller has a reserve price higher or lower than $\pi^p_b(k^{max}_b)$. In the first case, $p_b$ is matched with an agent and would have not be matched with it if it was honest, but it pays more than it wants; in the second it would have been matched with this agent anyway.
 \item[Case 3: $\pi_b \in \mathbb{R} \wedge \pi^p_b(k^{max}_b) \geq \pi_b$],  the bargaining is accepted at a price lower than the true reserve price of $p_b$. This is the only case where $p_b$ gets a good at an acceptable price. However, assuming the seller honesty, this can only occurs if the true reserve price of $p_b$ is compatible with the reserve price of  $s_b$. Hence, $s_b$ would have been matched with $p_b$ even if $p_b$ had been honest. 
\end{description}

\end{proof}

\Cref{theo:payment:bound} shows that  the proposals also provide a strong incentive to find an agreement while negotiating. 

\begin{theorem}\label{theo:payment:bound}
Rejecting an offer or going to a deadline is less interesting for a couple of 
rational negotiating agents than finding any agreement more advantageous than the declared reserve price, before the negotiation ends.
 \end{theorem}
\begin{proof}
Suppose that the bargaining $b$ failed. The utility of the purchaser is $u_p = (\pi_p^{matching}, k^s_b, k^p_b)$ where $\pi_p^{matching}$ is the reserve price that the purchaser has declared. According to \Cref{eq:rational}, $p$ would prefer any situation with a success on a price inferior than $\pi_p^{matching}$.
The same goes for the seller.
\end{proof}

\section{Effect of the introduced extensions on the curiosity}
\label{Sec:XP}

In this section we propose an empirical analysis of the impact of the introduction of the proposed extensions.
Our objective is to measure how an inappropriate collect of information is restrained. Restraining the collect of information without impacting the protocol performance ideally reduces the welfare of curious agents, increases the one of secretive agents and does not affect uncurious agents. 

We first describe the experimental protocol. Then,  we present empirical results about the variation of the number of exchanged messages and the agent welfare.

\subsection{Experimental Protocol}

Our experiments consist in running different instances of bargaining with randomly generated couples of agents.
A bargaining instance is identified to the possible variants: 
standard bargaining, with matching, with a given bound 
or with the three extensions altogether.  
The bound on the number of messages, $k^{max}_b$, is set to 500. This value has been computed from a set of experiments not presented here for reasons of space. It is the most advantageous for secretive agents on our specific experimental context: below this value, we observed that the bound strongly decreased their welfare, whereas above this value it remained constant. The same phenomenon was observed for uncurious and curious agents with, respectively, the values 750 and 1000. Those values correspond to the point where most of the deals were already ended. 

An agent $a$ is associated to a type (\textit{uncurious}, \textit{curious} or \textit{secretive}), 
an initial reserve price $\pi_a(0)$, a utility function $u_a$, and a strategy function.
The utility function of the agent is 
represented\footnote{Only the utility of purchaser is given here, the utility
of seller being similar} in \Cref{util}. One can verify the validity of those functions (see \Cref{fig:reservepr:all}) {\it w.r.t.} \Cref{def:uncurious,def:secretive,def:curious} .
\begin{equation} \label{util}
\begin{aligned}
&\forall b \in \mathcal B, u_a(\pi_{b}, \nu^s_{b},\nu^p_{b}) =
&\begin{cases}
\pi_a^b(0) \cdot \Delta(n)- \pi_b & \text{ if the negotiation succeeds}\\
\pi_a^b(0) \cdot\left(\Delta(n) - 1\right) & \text{ else}
\end{cases}\\
&\text{where}\\
&&\begin{cases}
\Delta(n) = 1 & \text{if $a$ is uncurious} \\
\Delta(n) = \dfrac{1}{\log_n(n+k_{b}^p)} & \text{if $a$ is secretive} \\
\Delta(n) = \log_n(n+k_{b}^p) & \text{if $a$ is curious} \\
\end{cases}
\end{aligned}
\end{equation}

A strategy function indicates  at each step the proposal the agent is willing to make. 
If the other party makes a better proposal, the agent accepts it. 
If the planned proposal is worse than the agent reserve price, the agent drops off the bargaining.
The used strategy is represented in \Cref{eq:strategy} and is an adaptation of \cite{Faratin1998Negotiation}. 
The parameters $\kappa_a$, $\beta_a$ and $\gamma_a$ determine the shape of the curve for agent $a$.
An example is plotted in  \Cref{fig:reservepr:all}.
\begin{equation}\label{eq:strategy}
    \gamma + \left(\kappa_a + (1-\kappa_a) \cdot \left(\dfrac{k}{k^{max}_a}\right)^{\frac{1}{\beta_a}}\right)\cdot (\gamma_a-\pi^b_a)
\end{equation}

An experiment considers a couple of agents, defined by their type and an instance of bargaining.
The agent reserve prices have been drawn from a Gaussian distribution and  $\kappa_a$, $\beta_a$ and $\gamma_a$ 
from uniform distributions.
Each experiment is averaged over 10'000 draws and we observe the final welfare of each type of agent.

\subsection{Impact on the welfare}

\Cref{fig:wlf} displays the welfare of the different types of agents in different settings. ``barg'' stands for standard bargaining, ``mat.'' for the matching extension, ``bou.'' for the bounding extension (with a bound of 500 messages) and ``all'' for the three extensions all together.
\begin{figure}[h!]
    \centering
    \subfloat[Welfare of the uncurious agents,\newline uncurious vs uncurious]{
        \centering
        \begin{tikzpicture}
\begin{axis}[
    ybar, 
    xtick=data,
    ymin=0, 
    xticklabels from table={cur-vs-unc.txt}{name},
        xticklabel style={align=center},
x tick label style={/pgf/number format/1000 sep=},
ylabel=Welfare,
bar width=5pt,
height = \personnalheight,
width = \personnalwidth,
enlarge x limits=0.25,
ylabel near ticks,
minor y tick num=3,
]
\addplot[fill=blue!50] table[x=number, y=welfare
    ,] {unc-vs-unc.txt};
\end{axis}
\end{tikzpicture}
        \label{fig:unc:wlf}
    }\hfill
    \subfloat[Welfare of the curious agents]{
        \centering
        \tikzset{hatch thickness/.store in=\hatchthickness,
        hatch thickness=10pt}
\begin{tikzpicture}
\begin{axis}[
    ybar, 
    xtick=data,
    ymin=0, 
    ymax=7,
    xticklabels from table={cur-vs-unc.txt}{name},
        xticklabel style={align=center},
x tick label style={/pgf/number format/1000 sep=},
ylabel=Welfare,
bar width=5pt,
height = \personnalheight,
width = \personnalwidth,
enlarge x limits=0.25,
ylabel near ticks,
minor y tick num=3,
]
\addplot[fill=blue!50] table[x=number, y=welfare
    ,] {cur-vs-unc.txt};
\addplot[pattern=north west lines,pattern color=red] table[x=number, y=welfare
    ,] {cur-vs-sec.txt};
\end{axis}
\end{tikzpicture}
        \label{fig:cur:wlf}
    }\\
    \subfloat[Welfare of the secretive agents]{
       \centering
       \begin{tikzpicture}
\begin{axis}[
    ybar, 
    xtick=data,
    ymax=6,
    xticklabels from table={sec-vs-unc.txt}{name},
        xticklabel style={align=center},
x tick label style={/pgf/number format/1000 sep=},
ylabel=Welfare,
enlargelimits=0.15,
bar width=5pt,
height = \personnalheight,
width = \personnalwidth,
enlarge x limits=0.25,
ylabel near ticks,
minor y tick num=3,
]
\addplot[fill=blue!50] table[x=number, y=welfare
    ,] {sec-vs-unc.txt};
\addplot[pattern=north west lines,pattern color=red] table[x=number, y=welfare
    ,] {sec-vs-cur.txt};
\end{axis}
\end{tikzpicture}
       \label{fig:sec:wlf}
    }\hfill
    \subfloat{
       \centering
    \begin{tikzpicture}
        \begin{customlegend}[ybar,legend columns=1,legend style={align=left,column sep=2ex},legend entries = {* vs uncurious,secretive vs curious}]
        \addlegendimage{fill=blue!50}
        \addlegendimage{pattern=north west lines,pattern color=red}
        \end{customlegend}
\end{tikzpicture}}\hfill
    \caption{Welfare of the agents in different situations}
    \label{fig:wlf}
\end{figure}

\Cref{fig:unc:wlf} shows how the proposed extensions affect the welfare of classical agents. 
Bounding slightly reduces their welfare (10\%): it tends to close the longest negotiations by rejects.
{\it A contrario}, all others increase noticeably their welfare.
Matching doubles it by increasing the number of successful negotiations.
Applying all extensions quadruples it: the agents make concessions more quickly in order to reach their reserve price
before the deadline (\Cref{theo:payment:bound}).

\Cref{fig:cur:wlf} shows the impact on curious agents.
First, we observe that their welfare is worse when they negotiate with secretive agents rather than uncurious ones.
Indeed, secretive agents tend to drop off the bargaining more often. 
Whatever the opponent, matching and applying all extensions lower the welfare by a sixth since agents get to an agreement more quickly. {\it A contrario}, using only bounding lowers the welfare by 10\% when negotiating with uncurious agents while it increases it by the same factor when negotiating with secretive ones. Please note that (1) bounding alone results in more reject, thus curious agents get more often information ``for free'' and (2) curious agents might accept overpriced deal for the sake of information collect. This last point is more critical when facing secretive agents since they get more tough as the bargaining is going on. Consequently, if they were to succeed, the deals rejected thanks to the bound would have been more disadvantageous for curious agents when they are facing secretive agents than with uncurious ones.

\Cref{fig:sec:wlf} shows the impact of the extensions on secretive agents.
Each extension has a positive impact. Please note that the average welfare
of secretive agents is negative with the standard protocol because of the negotiations that end by a reject: 
secretive agents have revealed information without counterpart. Matching limits considerably 
the number of rejects, giving thus a positive welfare to them with nearly the same absolute value. 
Bounding only slightly increases the welfare of secretive agents (its absolute value is reduced by 9\%)
by reducing the number of exchanged messages before a reject.
Last, adding the enforcement of transaction results in a really higher welfare (three times in absolute value) 
for the same reason as uncurious agents. The welfare of secretive agents is always worse when they negotiate with curious agents.

\paragraph{}

The extensions that we have introduced reach their goal by preventing curiosity. 
Both matching agents on their reserve prices and forcing them to make deals effectively make opponent modeling much harder. However, bounding the number of exchanged messages alone is not very beneficial.

\section{Conclusion and Perspectives}

\label{Sec:Concl}

This article studied curiosity in bargaining. Curious agents exploit information leaked during a bargaining in order to model their opponents and, thus, gain a significant advantage in future bargaining. In order to reduce it, we first introduced an original rationality model that considers the exchanged information. This led to the establishment of three types of agents: uncurious agents that correspond to the classical agents -- they do not consider the exchanged information; secretive agents that are concerned by the information they might unveil; and curious agents that seek for information leaked by their opponents. We then proposed three extensions of the standard bargaining protocol both to limit the exchanged information and to provide incentives to not seek it. Incentives have been studied theoretically and the general impact of those extensions has been analyzed empirically.

Our experiments showed how strongly curiosity impacts the bargaining process. Indeed, introduction of curious agents results in longer negotiation and higher prices, while it is the contrary for secretive ones. Used simultaneously, our three extensions have proven to be really effective for preventing agents from seeking information, assuming they declare their true intention. However, this extended bargaining protocol can be manipulated by agents that are willing to negotiate less often but in more favorable conditions.

We plan to enrich our proposal in future work by addressing this last issue in the context of one-to-many negotiation. Indeed, the competitive environment of, for instance, Iterated Contract Net Protocol would help to deal with this issue.

\bibliographystyle{plain}
\bibliography{bibli}
\end{document}